\documentclass[10pt,a4paper]{article}

\usepackage[T1]{fontenc}
\usepackage[utf8]{inputenc}
\usepackage{lmodern}
\usepackage{microtype}
\usepackage[margin=1in]{geometry}
\usepackage{amsmath,amssymb,amsthm,mathtools}
\usepackage{booktabs}
\usepackage{array}
\usepackage{tabularx}
\usepackage{multirow}
\usepackage{enumitem}
\usepackage{xcolor}
\PassOptionsToPackage{hyphens}{url}
\usepackage[colorlinks=true,linkcolor=blue!50!black,urlcolor=blue!60!black,citecolor=blue!60!black]{hyperref}
\usepackage{caption}
\usepackage{url}
\usepackage{graphicx}
\usepackage{adjustbox}
\usepackage{float}
\usepackage{titlesec}

\setlist[itemize]{leftmargin=1.4em,topsep=2pt,parsep=2pt,itemsep=2pt}
\setlist[enumerate]{leftmargin=1.6em,topsep=2pt,parsep=2pt,itemsep=2pt}
\titlespacing*{\section}{0pt}{1.5\baselineskip}{0.5\baselineskip}
\titlespacing*{\subsection}{0pt}{1.0\baselineskip}{0.3\baselineskip}
\titlespacing*{\paragraph}{0pt}{0.6\baselineskip}{0.8em}

\theoremstyle{plain}
\newtheorem{theorem}{Theorem}
\newtheorem{proposition}[theorem]{Proposition}

\theoremstyle{definition}

\newcommand{\PCT}{\textsc{PCT}}
\newcommand{\cell}[1]{\texttt{#1}}
\renewcommand{\Re}{\operatorname{Re}}
\renewcommand{\Im}{\operatorname{Im}}
\let\origpm\pm
\renewcommand{\pm}{\ensuremath{\,\origpm\,}}

\title{\textbf{Complex-valued Phase-Coherent Transformers:\\[2pt] \large Experiments for the Conditions of a General-Purpose Complex Transformer}}
\author{Leona Hioki\\ \texttt{leohio@pm.me}}
\date{}

\begin{document}
\maketitle

\begin{abstract}
\noindent
Complex-valued Transformers have inherited softmax attention over the raw complex inner
product. Outside natively complex domains this standard form stays near chance, and no
complex attention had been shown to correct it. We show that the match must be a scaled
cosine score: L2-normalise queries and keys, so the score reads their cosine similarity and
ignores their magnitudes, and hold that score at order-one scale. With this the same models
train on four diagnostic tasks under two different gates; without the normalisation they
stay at chance on ListOps and Needle under both gates and fall far below on the other two,
and a normalised score placed at too small a scale fails as well. The resulting family of
phase-coherent Transformers (\PCT) matches or exceeds the strongest real-valued baseline
across long-range memory, positional retrieval, hierarchical reasoning, frequency-domain
classification and physical complex signals; it
shows no degradation up to depth 20; and its loss decreases log-linearly over a 61-fold
range of parameters. A member of the family,
complex screening combined with a phase-coherent recurrence, is the first genuinely
complex-valued neural network to solve Path-X, with 91.6\% of its trainable parameters
complex-valued against 38.2\% for S4. We record these as signs of generalisation not
previously seen in complex-valued neural networks.
\end{abstract}

\section{Introduction}

Complex-valued neural networks (CvNNs) were developed for domains whose signals are
natively complex: wireless communication, radar, MRI and audio. Complex extensions of
convolutions, normalisation layers and activations have been studied systematically; the
Transformer side is thinner. Existing complex attention mechanisms, from the component-wise
attention of Yang et al.\ (2020) to the complex scaled dot-product attention of Eilers and
Jiang (2023) and the phase-interference attention of the Holographic Transformer (Hao et
al., 2025), keep the row-normalised weighting of real-valued attention, and their score
paths differ from one another in normalisation and scale. Which property of the score path
decides whether a complex Transformer trains has not been isolated.

The gap shows in performance. The standard form of a complex Transformer, softmax over the
raw complex inner product, stays near chance on discrete tasks outside natively complex
domains, 0.10 on Copy Memory (delay 2{,}000), 0.10 on ListOps (length 1{,}024) and 0.00 on
Needle-in-a-Haystack (length 2{,}048), and falls far behind the strongest real-valued
attention at the same parameter budget, which solves the first and the third and reaches
0.70 on the second (Section~\ref{sec:breadth}). Replacing the softmax by an element-wise
sigmoid does not help (0.15, 0.13 and exactly chance, $1/64$, in the isolation setting of
Section~\ref{sec:attention}), so the problem is not the gate but the raw complex score.

This paper shows that the problem disappears when the match is a scaled cosine score.
Normalise queries and keys to unit norm and score them by
$s_{ij}=\Re\langle\bar q_i,\bar k_j\rangle\sqrt{d_h}$: dividing each complex vector by its
modulus makes the score their cosine similarity, which reads the direction of the two
vectors (their phase pattern and the relative amplitudes of the components) and ignores how
large either vector is, and the $\sqrt{d_h}$ factor places it at order-one scale. We call a layer that applies a
real-valued gate to this scaled cosine score a \emph{phase-coherent attention layer}, and the
family defined by complex-linear projections, the scaled cosine score and a real gate that is a
function of the score matrix the \emph{\PCT{} family}. Its canonical member uses an
element-wise sigmoid gate (\PCT); an L2-softmax and a screening gate belong to the same
family. In an isolation experiment that toggles only the normalisation, both the sigmoid and
the row-softmax gate reach 1.000 on Copy, 1.000 on Needle, 0.50--0.53 on ListOps and
0.50--0.51 on FFT-MNIST with the normalisation; without it both stay at chance on ListOps
and Needle, the sigmoid gate stays at chance on Copy, and FFT-MNIST falls to 0.38
(Section~\ref{sec:attention}). The mechanism is measurable at initialisation:
unnormalised scores are constant within each row to within $10^{-4}$, so attention starts
uniform and the selection gradient is vanishingly small; normalisation restores an $O(1)$ spread, and in
the complex case that spread is the direction of the vectors, their phase pattern and
relative amplitudes. The only further condition is that the
gate passes gradient at initialisation; conditions on the \emph{shape} of the gate are
deferred to Section~\ref{app:gates}.

The family works broadly outside complex domains. Across long-range memory (Copy),
positional retrieval (Needle), hierarchical reasoning (ListOps), frequency-domain
classification (FFT-MNIST), byte-level text and pixel-level image classification, and
modulation classification on raw I/Q signals (RadioML), it matches or exceeds real
screening, the strongest real-valued baseline at matched parameter budget
(Section~\ref{sec:breadth}). Accuracy on ListOps does not degrade from depth 2 to depth 20,
and growing width and depth together at a fixed ratio from 0.10M to 6.3M parameters lowers
evaluation loss log-linearly ($-0.081$ per decade, $R^2=0.999$). The depth degradation
feared for complex substrates is not observed.

Finally, combining the family's complex screening gate with a phase-coherent linear
recurrence solves Path-X (length 16{,}384) at $92.71\pm0.89$ over three seeds, above S4-v1
and DSS and between S4D-LegS and S4D-Inv (Section~\ref{sec:pathx}). This is the
first solution of Path-X by a genuinely complex-valued neural network, in a measurable
sense: 91.6\% of its trainable parameters are complex-valued and complex arithmetic carries
both the transport path (the recurrence) and the matching path (the screening attention),
whereas S4 at its published Path-X configuration has 38.2\%, all inside the state-space
kernel; in both models the inter-block residual stream is real. Fixing the recurrence
eigenvalues to be real (phase zero) never leaves chance, so phase is necessary.

\textbf{Contributions.}
\begin{enumerate}
  \item Query--key normalisation is a precondition for training complex attention across tasks: a
  one-variable isolation over four tasks and two gates, with the mechanism (row-constant
  scores at initialisation) measured directly.
  \item The \PCT{} family, defined by the scaled cosine score, within which the element-wise
  sigmoid, L2-softmax and screening gates are task-dependent siblings of equal standing.
  \item Depth robustness to depth 20 and log-linear scaling over a 61-fold parameter range.
  \item A Path-X solution by complex screening with a phase-coherent recurrence, with the
  measured complex-valued parameter share (91.6\% against 38.2\% for S4) and a
  phase-necessity ablation.
\end{enumerate}

The main text is limited to these results. The reduction of gate-shape conditions to
initialisation, the mechanism measurements, the machine-checked equivariance and
depth-uniform stability of the family (Lean~4), and the changes from the previous version
of this manuscript are in the appendices. The results are obtained at one configuration per
task on small-scale diagnostic tasks in a solve-versus-chance regime; we do not rank the
surviving variants on real datasets.

\section{Phase-coherent attention}
\label{sec:attention}

\subsection{Definition}
\label{sec:definition}

A phase-coherent attention layer maps token states $x_i\in\mathbb{C}^{d}$ through
complex-linear projections $W_q,W_k,W_v,W_o$:
\begin{align*}
q_i &= W_q x_i, \qquad k_j = W_k x_j, \qquad v_j = W_v x_j,\\
\bar q_i &= q_i/\lVert q_i\rVert, \qquad \bar k_j = k_j/\lVert k_j\rVert,\\
s_{ij} &= \Re\langle \bar q_i,\bar k_j\rangle\sqrt{d_h}\in[-\sqrt{d_h},\sqrt{d_h}],\\
o_i &= W_o \sum_{j} \alpha_{ij}\,v_j, \qquad \alpha = g(S)\in\mathbb{R}^{N\times N},
\end{align*}
where $d_h$ is the head dimension and $S=(s_{ij})$ the score matrix. The normalisation makes
$s_{ij}$ a cosine similarity: it removes the magnitude of each vector and leaves the
direction, that is the per-component phase differences weighted by the relative amplitudes;
by Cauchy--Schwarz it lies in $[-\sqrt{d_h},\sqrt{d_h}]$, and the $\sqrt{d_h}$ factor places
it at order-one scale. Three properties define the \PCT{} family: (P1) complex-linear
projections, (P2) the scaled cosine score, and (P3) a real-valued gate $g$ that is a
function of the score matrix. Its members differ only in $g$:
\begin{itemize}
  \item \textbf{\PCT{} (canonical)}: the element-wise sigmoid $\alpha_{ij}=\sigma(s_{ij}+b)$
  with $b$ initialised to $-\log N$ (Ramapuram et al., 2025), applied without row
  normalisation.
  \item \textbf{L2-softmax}: $\alpha_{ij}=\exp(s_{ij})/\sum_k\exp(s_{ik})$, a row softmax on the normalised score.
  \item \textbf{Screening}: $\alpha_{ij}=r^2\,\mathrm{relu}(s_{ij}-t)^2$ with a learnable
  width $r$ and threshold $t$ per head, followed by TanhNorm and a modReLU gate on the
  aggregate (the complex lift of Nakanishi, 2026).
\end{itemize}
Since real weights multiply complex values, every value contributes to $o_i$ with its
phase unchanged. Since the score matrix is invariant under a global rotation
$x_i\mapsto e^{i\varphi}x_i$ of the input, every member of the family commutes with global
phase rotation; the proof is a cancellation and is machine-checked
(Appendix~\ref{app:lean}). The real-valued counterparts replace the projections by real
matrices and the score by the cosine of real vectors; the standard real baselines are
softmax (Vaswani et al., 2017), sigmoid (Ramapuram et al., 2025) and screening (Nakanishi,
2026). The standard complex baseline, written \cell{complex\_softmax} throughout, is the
row softmax over $\Re\langle q_i,k_j\rangle/\sqrt{2d_h}$ of the \emph{unnormalised} complex
queries and keys, as ported by lucidrains (2024) from Eilers and Jiang (2023); it differs
from the L2-softmax member of the family in exactly one respect, the normalisation.

\subsection{A scaled cosine score is necessary for training across tasks}
\label{sec:normalisation}

We ablate the L2 normalisation of queries and keys, holding everything else fixed. To
separate the effect of normalisation from the choice of gate, we run the ablation under two
gates: the element-wise sigmoid, in which keys do not compete, and the row softmax, in which
keys compete for a unit budget. Four diagnostic tasks with known chance levels are used:
\emph{Copy Memory} (reproduce ten tokens after a delay of 1{,}000; chance 0.10),
\emph{Needle-in-a-Haystack} (retrieve one token from a length-2{,}048 sequence of
same-vocabulary distractors; chance $1/64\approx0.016$), \emph{ListOps} (hierarchical list
evaluation at length 1{,}024; chance 0.10) and \emph{FFT-MNIST} (classify the two-dimensional
Fourier transform of MNIST digits; chance 0.10). Copy, ListOps and FFT use models of 0.8M
parameters (dim 128, depth 4) trained for 2K steps; Needle uses models of 4.7M complex
parameters (dim 256, depth 6) trained for 15--30K steps. All numbers are means over three seeds with a
deterministic final evaluation over 2{,}048 samples (1{,}024 for Needle); $\pm$ denotes the
standard deviation across seeds.\footnote{FFT-MNIST accuracy is measured on samples drawn
from the training tensors, a limitation of the data pipeline shared by all variants; its
values are comparable across columns but are not held-out claims.}

\begin{table}[H]
\centering\small
\begin{tabular}{lcc}
\toprule
task (chance) & sigmoid gate & row-softmax gate \\
\midrule
Copy Memory (0.10)   & \textbf{1.000\pm0.000} / 0.151\pm0.017 & \textbf{1.000\pm0.000} / 0.820\pm0.183 \\
ListOps (0.10)       & \textbf{0.498\pm0.016} / 0.134\pm0.006 & \textbf{0.526\pm0.007} / 0.134\pm0.006 \\
Needle (0.016)       & \textbf{1.000\pm0.000} / 0.016\pm0.001 & \textbf{1.000\pm0.000} / 0.000$^{\dagger}$ \\
FFT-MNIST (0.10)     & \textbf{0.498\pm0.003} / 0.381\pm0.018 & \textbf{0.507\pm0.026} / 0.384\pm0.013 \\
\bottomrule
\end{tabular}
\caption{Accuracy with / without query--key normalisation, under two gates. $^{\dagger}$Single
30K-step run at this configuration; fifteen further runs of the same unnormalised cell at other
lengths and budgets all score $\le0.25$.}
\label{tab:norm}
\end{table}

Without normalisation both gates are at chance on ListOps and Needle, the sigmoid gate is
at chance on Copy and the softmax gate is unstable there ($0.61$--$0.95$ across seeds), and
both fall to $0.38$ on FFT-MNIST; with it, both solve Copy and Needle deterministically and
reach the same ListOps accuracy. The pattern does
not depend on whether keys compete, so an explanation of these results based on the
softmax's competition for attention mass, including the one given in the previous version
of this manuscript, is ruled out by the sigmoid column.

\paragraph{Why the unnormalised models do not train: an optimisation effect, not an expressive one.} The
unnormalised model class contains the normalised solutions (the key projection could learn
norm-equalised keys), so the difference must arise in training, and it arises at the
earliest possible point. At initialisation, unnormalised scores are constant within each
attention row to within $10^{-4}$ (row standard deviation $6.5\times10^{-5}$ on embedded
ListOps inputs), so softmax attention starts exactly uniform (row entropy $=\ln N$ to
machine precision) and the sigmoid gate starts at a constant weight. A row that is constant to $10^{-4}$ yields a selection gradient of the same order, and the
models do not leave this state: after 30K steps their attention
is still uniform in every layer and their key norms have equalised (coefficient of variation
0.003) without any selective structure. Two controls confirm the account. Rescaling the
unnormalised score by a temperature $\tau\in\{0.25,1,4\}$ leaves training trajectories
numerically identical, since a constant row is constant at every temperature. Normalisation,
in turn, restores an $O(1)$ score spread at the same initialisation (row standard deviation
0.61), and in the complex case that spread is the direction of the vectors (their phase
pattern and relative amplitudes), because dividing a complex vector by its modulus preserves
the argument of every component
(Section~\ref{app:mechanism}).

\paragraph{Relation to prior work.} Query--key normalisation is an established stabiliser in
real-valued Transformers: it bounds attention logits in low-resource translation (Henry et
al., 2020), enables scaled cosine attention in Swin~V2 (Liu et al., 2022), prevents logit
divergence in ViT-22B (Dehghani et al., 2023), and nGPT (Loshchilov et al., 2024)
normalises all representations. In the complex domain, the phase reading of
$\Re\langle q,k\rangle$ is discussed by Eilers and Jiang (2023), whose architecture retains
magnitudes, and a norm-divided complex score is one component of the Holographic
Transformer (Hao et al., 2025), which adds a phase-interference decay and a phase rotation
of the values on top of it. What Table~\ref{tab:norm} adds is the causal statement: in
complex attention, a cosine score at order-one scale is not a stabiliser but a precondition
for training across tasks, chance versus solved across four tasks and two gates; the normalisation
supplies the cosine and the $\sqrt{d_h}$ factor supplies the scale, and both are needed. A
Holographic-style variant that adds the phase machinery on top of the normalised score, at
our score scale, trains to parity with the L2-softmax in our harness
(Section~\ref{app:variants}).

\paragraph{One axis, two ends.} Any query--key score factorises as
$s_{ij}=\lVert q_i\rVert\,\lVert k_j\rVert\cos\angle(q_i,k_j)$: the match is decided by the
angle, and $\lVert q\rVert\lVert k\rVert$ sets its scale. The failure above is the small-scale
end of this axis: at standard initialisation the product is of order $10^{-4}$, the rows are
flat, and attention starts uniform. The large-scale end is the one reported for real-valued
models at scale, where growing norms drive the logits up until attention entropy collapses
and training diverges (Dehghani et al., 2023; Zhai et al., 2023). Training requires the score
to stay in an $O(1)$ window between the two, and query--key normalisation is the operation
that pins it there: dividing by $\lVert q\rVert\lVert k\rVert$ is a per-pair temperature that
follows both the initialisation scale and the drift of the norms during training, and leaves
the angle as the only carrier of the match. This is also why a fixed temperature cannot
substitute for it: the value required at initialisation is of order $10^{4}$ and changes as
the norms move. In the complex case the angle is the phase structure of the two vectors,
which is what the rest of this paper is about.

\subsection{The gate is nearly free}
\label{sec:gate}

Given the scaled cosine score, how much does the gate matter? Table~\ref{tab:family}
compares the three members of the family at one configuration per task.

\begin{table}[H]
\centering\small
\begin{tabular}{lcccc}
\toprule
gate & Copy & ListOps & FFT-MNIST & Needle \\
\midrule
sigmoid (\PCT)  & 1.000\pm0.000 & 0.498\pm0.016 & 0.498\pm0.003 & 1.000\pm0.000 \\
L2-softmax      & 1.000\pm0.000 & \textbf{0.526\pm0.007} & 0.507\pm0.026 & 1.000\pm0.000 \\
screening       & 1.000\pm0.000 & 0.506\pm0.011 & \textbf{0.693\pm0.060} & 1.000$^{\ddagger}$ \\
\bottomrule
\end{tabular}
\caption{The three gates of the family on the scaled cosine score (same configurations as
Table~\ref{tab:norm}). $^{\ddagger}$Single seed at this configuration.}
\label{tab:family}
\end{table}

Three observations. First, the L2-softmax matches or exceeds the default everywhere, and on
Needle it solves faster (accuracy 1.0 from step $\approx$2{,}000 against
3{,}000--6{,}000). This is explained by a factorisation: L2-softmax weights decompose as an
element-wise exponential gate times a per-query \emph{positive real} scalar, so the mixture
of values, and hence the phase structure of the output, is that of a non-competing gate, and
competition only modulates a phase-inert gain. Second, screening leads on FFT-MNIST by
$+0.19$ and the L2-softmax on ListOps; at depth 20 the L2-softmax reaches
$0.684\pm0.002$ on ListOps, above the default at every depth we tested
(Section~\ref{sec:depth}). Differences among the members are task-dependent rather than
ordered, and we make no claim that the sigmoid gate is optimal. Third, the one property the
gate must have is an operating point that passes gradient at initialisation. With the
standard bias $b=-\log N$, gates that are flat on the negative axis (ReLU and its clamped
form) have nearly their whole score range inside the dead region: on Copy the range is
entirely negative, they receive exactly zero gradient and they stay at chance, while on the
shorter FFT-MNIST sequence about one per cent of the range survives and they reach $0.31$
against the sigmoid's $0.50$. Moving the bias to $0$, with nothing else changed, brings them
to $1.000$ on Copy and to $0.62$--$0.67$ on FFT-MNIST, and the same holds for squared and
cubed gates (Section~\ref{app:init}). No shape property of the gate, whether smoothness, boundedness or
growth order, survives as a necessary condition once the operating point is alive. What the
sigmoid does provide is robustness: its gradient is nowhere zero, so it is the only gate in
our record that needs no bias tuning, and it holds full accuracy across the learning-rate
and batch windows we tested (Appendix~\ref{app:robustness}).

The score path also separates the complex family from its real counterpart in one
controlled setting. Applying the same normalisation to real-valued sigmoid and softmax cells
(cosine scores of real vectors) solves Copy at $1.000$ but leaves Needle at $0.26$--$0.32$
over three seeds each, while both complex members reach $1.000$
(Section~\ref{app:realnorm}, Table~\ref{tab:realnorm}). Since real screening does solve
Needle, this is a difference between score paths at fixed gate type, and we draw no
field-level claim from it; together with the phase-necessity ablation of
Section~\ref{sec:pathx}, it is the evidence that the word \emph{phase} in this paper is
load-bearing rather than decorative.

\subsection{The gate must pass gradient at initialisation}
\label{app:init}\label{app:gates}

The experiments of this and the following three subsections use the configurations of
Table~\ref{tab:norm} (Copy $d{=}1{,}000$, ListOps and FFT-MNIST at dim 128, depth 4, 2K
steps; Needle at dim 256, depth 6), three seeds, and a deterministic evaluation over
2{,}048 samples (1{,}024 for Needle). FFT-MNIST is measured on training-tensor samples
throughout.

With normalised scores the gate input $s_{ij}+b$ is confined to
$[-\sqrt{d_h}+b,\sqrt{d_h}+b]$, and the standard bias $b=-\log N$ places that interval
according to the sequence length. On Copy it lies entirely below zero
($\sqrt{32}-\log 1024=-1.27$), so gates that are flat on the negative axis are identically
zero with identically zero gradient, in every parameter (Appendix~\ref{app:lean}); this is
the configuration their chance-level readings come from. On FFT-MNIST, whose 256 tokens give
a smaller bias, the upper end is $\sqrt{32}-\log 256=+0.11$, about one per cent of the
score range: the dead region no longer covers everything, and the two ReLU gates read
$0.314$ rather than chance, still $0.18$ below the sigmoid. Table~\ref{tab:init} repeats
training with the bias moved to $0$ and nothing else changed.

\begin{table}[H]
\centering\small
\begin{tabular}{lcc}
\toprule
gate & Copy: $b{=}-\log N \to b{=}0$ & FFT-MNIST: $b{=}-\log N \to b{=}0$ \\
\midrule
ReLU              & 0.102\pm0.001 $\to$ \textbf{1.000\pm0.000} & 0.314\pm0.014 $\to$ \textbf{0.674\pm0.033} \\
clamped ReLU      & 0.102\pm0.001 $\to$ \textbf{1.000\pm0.000} & 0.314\pm0.014 $\to$ \textbf{0.618\pm0.095} \\
squared ($s^2$)   & 0.998\pm0.003 $\to$ 1.000\pm0.000          & 0.219\pm0.010 $\to$ \textbf{0.589\pm0.074} \\
cubed ($s^3$)     & 1.000\pm0.000 $\to$ 1.000\pm0.000          & 0.228\pm0.015 $\to$ \textbf{0.485\pm0.094} \\
sigmoid (control) & 1.000\pm0.000 $\to$ 1.000\pm0.000          & 0.498\pm0.003 $\to$ 0.516\pm0.071 \\
\bottomrule
\end{tabular}
\caption{Accuracy before and after moving the gate bias to $0$.}
\label{tab:init}
\end{table}

Every gap between the left-hand and right-hand columns closes under a one-line change of
initialisation. No shape property of the gate (smoothness, boundedness, growth order)
survives as a necessary condition once the operating point is alive; hypotheses one might
entertain about unbounded gates under superposition, including one stated in the previous
version of this manuscript, are explained by initialisation. The sigmoid is the only gate in
this table that never required the care: its gradient is nowhere zero, so it is insensitive
to the bias choice.

\subsection{Gate variants}
\label{app:variants}

Table~\ref{tab:variants} collects further variants, each changing one property of the
default gate while keeping the scaled cosine score.

\begin{table}[H]
\centering\footnotesize
\setlength{\tabcolsep}{4pt}
\begin{tabular}{llccc}
\toprule
variant & property changed & Copy & ListOps & FFT-MNIST \\
\midrule
sigmoid (\PCT)              & ---                          & 1.000\pm0.000 & 0.498\pm0.016 & 0.498\pm0.003 \\
L2-softmax                  & keys compete ($\sum_j\alpha_{ij}=1$) & 1.000\pm0.000 & 0.526\pm0.007 & 0.507\pm0.026 \\
screening                   & multiplicative screening     & 1.000\pm0.000 & 0.506\pm0.011 & 0.693\pm0.060 \\
complex-valued gate$^{a}$   & gate rotates each contribution & 1.000\pm0.000 & 0.298\pm0.043 & 0.471\pm0.039 \\
complex-normalised softmax$^{b}$ & normaliser rotates tokens & 0.909\pm0.157 & 0.429\pm0.043 & 0.416\pm0.028 \\
Holographic-style$^{c}$ & phase-interference decay, value rotation & 1.000\pm0.000 & 0.483\pm0.016 & 0.536\pm0.073 \\
\bottomrule
\end{tabular}
\caption{Gate variants on the scaled cosine score. $^{a}\;\alpha_{ij}=\sigma(s_{ij}+b)\,e^{is_{ij}}$.
$^{b}$Weights $e^{z_{ij}}/\sum_k e^{z_{ik}}$ with complex logits
$z_{ij}=\langle\bar q_i,\bar k_j\rangle\sqrt{d_h}$. $^{c}$Row softmax over
$s_{ij}\,e^{-\alpha|\Delta\phi_{ij}|}$ with $\Delta\phi_{ij}=\arg\langle q_i,k_j\rangle$
and learnable $\alpha$, values rotated by $e^{i\Delta\phi_{ij}}$, at the score scale of
Section~\ref{sec:definition}; on Needle all three seeds reach $1.000$ by step 2{,}500 (in-loop
evaluation; the runs ended before the final evaluation).}
\label{tab:variants}
\end{table}

Even gates that touch phase, by rotating contributions or tokens, solve Copy and remain
well above chance elsewhere, at a cost on ListOps and with one unstable seed on Copy for
the complex normaliser. The cost of writing to phase depends on what is written: the
score-driven rotation costs $-0.20$ on ListOps, whereas rotation by the geometric phase
mismatch $\arg\langle q,k\rangle$ costs $-0.04$, within noise. Keeping all learned
modulation real is therefore the care-free choice rather than a requirement. The
Holographic-style variant, which adds its phase machinery on top of the normalised score,
neither helps nor hurts relative to the L2-softmax on any of the four tasks (Copy and Needle
tie at $1.000$, FFT $+0.03$, ListOps $-0.04$); the normalisation suffices alone. At its
published score scale ($1/\sqrt{d_k}$) the same variant, in our harness, stays at chance on
ListOps ($0.134\pm0.006$; Copy $1.000$, FFT-MNIST $0.534\pm0.072$), with the decay
coefficient $\alpha$ frozen at its initial value in every seed ($0.96$--$0.99$ per layer),
and trains ($0.483\pm0.016$) once the score is brought to the scale of
Section~\ref{sec:definition}. This is the row-flat initialisation of
Section~\ref{sec:normalisation}, reproduced on a design that was not used to derive it,
and the frozen $\alpha$ is an independent signature of the same state, since no gradient
reaches it. We do not claim that the published model fails in its own setting, whose
encoder, data and training we did not reproduce.

\subsection{Normalised real baselines}
\label{app:realnorm}

Applying the same query--key normalisation to the real-valued sigmoid and softmax cells
(cosine scores of real vectors) gives Table~\ref{tab:realnorm}.

\begin{table}[H]
\centering\small
\begin{tabular}{lccc}
\toprule
task & real sigmoid, normalised & real softmax, normalised & complex counterparts \\
\midrule
Copy $d{=}1000$   & 1.000\pm0.000 & 1.000\pm0.000 & 1.000 / 1.000 \\
Needle $L{=}2048$ & 0.317\pm0.170 (max 0.50) & 0.259\pm0.133 (max 0.40) & \textbf{1.000} / \textbf{1.000} \\
\bottomrule
\end{tabular}
\caption{Real-valued cells with normalised queries and keys ($N{=}3$).}
\label{tab:realnorm}
\end{table}

Normalisation is field-agnostic on Copy: the normalised real cells solve it, so
unnormalised real baselines understate real attention on retrieval, which is why
Section~\ref{sec:crosstask} reads the real comparison against real screening. On Needle the
normalised real cells improve from $0.000$ to $0.26$--$0.32$ but do not solve the task,
while both complex members reach $1.000$; among sigmoid- and softmax-gated cells this is the
one controlled setting where the complex score path is decisively stronger than its real
counterpart. Since real screening does solve Needle, the gap is specific to gate type, and
we advance no field-level claim from it.

\subsection{Mechanism measurements}
\label{app:mechanism}

\textbf{Scores at initialisation.} On embedded ListOps inputs, the unnormalised complex
score $\Re\langle q_i,k_j\rangle/\sqrt{2d_h}$ has a within-row standard deviation of
$6.5\times10^{-5}$ at initialisation, so the row softmax has entropy $\ln N$ to machine
precision and the sigmoid gate a constant weight per row. The normalised score
$\Re\langle\bar q_i,\bar k_j\rangle\sqrt{d_h}$ has a within-row standard deviation of
$0.61$ at the same initialisation. Any real scalar $\lambda>0$ applied to $q$ or $k$ leaves
the normalised score unchanged, and for a fixed direction the unnormalised score can be made
arbitrarily large by magnitude (Appendix~\ref{app:lean}); the normalised score therefore
carries exactly the directional information, the phase pattern and relative amplitudes of
the two vectors.

\textbf{Temperature.} Rescaling the unnormalised score by $\tau\in\{0.25,1,4\}$ before the
softmax leaves the ListOps training trajectories numerically identical
($0.1336\pm0.0058$ in all three arms, train loss equal to five decimals). A cell-level check
confirms that $\tau$ reaches the module and changes its outputs on synthetic input; on
task data the row is constant, and a constant row is constant at every temperature.

\textbf{After training.} At 30K steps the unnormalised models' attention is still uniform in
every layer, and their key norms have equalised to a coefficient of variation of $0.003$,
without any selective structure emerging. The earlier account of these results, in which
large-norm keys dominate the match, is not what the trained models show; the state is
uniform attention from the first step.

\textbf{Stability constants on trained models.} On the depth-2 to depth-20 \PCT{} models of
Section~\ref{sec:depth}, the end-to-end $\ell^2$ response to zero-mean per-token phase
perturbations of the input is $10.8$--$13.2$ with no trend in depth; output norms are
$23.3$--$26.7$; and the contraction ratio of the stack on pairs of inputs is $0.11$--$0.19$.
These are the constants $K$, $Y_{\max}$ and the non-expansiveness that
Theorem~\ref{thm:depth} takes as hypotheses.

\subsection{What the family comparison establishes}
\label{sec:familysummary}

Read together, Sections~\ref{sec:normalisation}--\ref{app:mechanism} sharpen the claim of
this paper in four ways.

\textbf{The condition is on the score, and it is one condition.} Removing the
normalisation collapses ListOps and Needle to chance under the sigmoid and under the
L2-softmax alike (Table~\ref{tab:norm}); rescaling the unnormalised score by a fixed
temperature changes nothing (Section~\ref{app:mechanism}); and a normalised score placed
at the published Holographic scale $1/\sqrt{d_k}$ collapses on ListOps in the same way,
with its learnable coefficient frozen (Section~\ref{app:variants}). The failing state is
the same in every case, uniform attention from initialisation, and the cure is the same,
a cosine score at order-one scale.

\textbf{The gate has one condition, not several.} Every shape hypothesis about the gate
(smoothness, boundedness, growth order, competition among keys) reduces to whether the gate
passes gradient at its operating point: the ReLU, clamped-ReLU, squared and cubed gates all
recover under a one-line change of bias, on both tasks measured (Table~\ref{tab:init}), and
the L2-softmax factorises into the element-wise gate times a positive per-token scalar
(Section~\ref{sec:gate}).

\textbf{Within the family, the gate orders the members by task.} The three members tie on
retrieval (Copy and Needle at $1.000$); the L2-softmax leads on ListOps and at depth 20
($0.684$ against $0.633$); screening leads on FFT-MNIST ($0.693$ against $0.498$); the
gates that write to phase pay only on the compositional task, by an amount set by what they
write (Table~\ref{tab:variants}). \PCT{}'s distinction is that it is the only member that
requires no bias or scale choice anywhere in this suite.

\textbf{The complex score path separates from its real twin on exactly one task.} With
the same normalisation and the same gate, real cells solve Copy as the complex cells do
and stay at $0.26$--$0.32$ on Needle where the complex cells reach $1.000$
(Table~\ref{tab:realnorm}); the comparison of Section~\ref{sec:breadth} is read against
real screening for this reason.

\section{Breadth, depth and scale}
\label{sec:breadth}

\subsection{Protocol}
\label{sec:protocol}

We compare six cells, $\{\text{real},\text{complex}\}\times\{\text{softmax},\text{sigmoid},
\text{screening}\}$, on nine benchmark rows. The complex cells are \cell{complex\_softmax}
(the unnormalised standard form of Section~\ref{sec:definition}), \PCT{} and complex
screening. Parameter budgets are matched by $\mathrm{real\_dim}=\mathrm{complex\_dim}\times
1.41\approx\sqrt2$: real cells use dim 184 and head dimension 46, complex cells dim 128 and
head dimension 32. A complex linear map of width $D$ stores $2D^2$ real scalars (real and
imaginary parts) and a real linear map of width $\sqrt2 D$ stores $2D^2$, so the two sides
hold the same number of trainable scalars; the same $2D^2$ count applies to the products
formed per token if each complex weight--activation product is counted once, as its real
and imaginary channels, while a strict count of real multiplications (four per complex
product) puts the complex side at $4D^2$ against $2D^2$. The mid-scale Copy and Needle rows use dim 256
and depth 6 for both families; there the complex cells hold roughly twice the real-scalar
count (9.48M against 4.74M), and those rows are equal-width rather than parameter-matched
comparisons. All rows use three seeds and batch 32; screening cells run without the cosine
softmask. Evaluation differs by row and is stated in the table: the LRA-Text and LRA-Image
rows are held-out test accuracies with a deterministic evaluation over 2{,}048 samples; the
synthetic rows draw fresh samples on the fly and report the final-step reading of the
original sweep (a single batch of 32 sequences); the Needle row is a single seed
for every cell except \PCT{}. RadioML uses the 6\,dB subset of the public RML2016 mirror at
dim 64 and depth 3. Full configurations are in Appendix~\ref{app:details}.

The real softmax and sigmoid rows are the standard, unnormalised forms. Normalising the real
queries and keys as well lifts real sigmoid and softmax to $1.000$ on Copy $d{=}1{,}000$
and to $0.26$--$0.32$ on Needle (Section~\ref{app:realnorm}); the informative real
comparison in every row is therefore real screening, which is the strongest real cell
throughout.

\subsection{Cross-task comparison}
\label{sec:crosstask}

\begin{table}[H]
\centering\footnotesize
\setlength{\tabcolsep}{4pt}
\begin{tabular}{llcccccc}
\toprule
task & evaluation & \cell{r\_softmax} & \cell{r\_sigmoid} & \cell{r\_screen} & \cell{c\_softmax} & \PCT & \cell{c\_screen} \\
\midrule
Copy $d{=}500$ (mid)   & synthetic & 0.445 & 0.945 & \textbf{1.000} & 0.092 & \textbf{1.000} & 0.528 \\
Copy $d{=}2000$ (mid)  & synthetic & 0.100 & 0.083 & \textbf{1.000} & 0.100 & \textbf{1.000} & 0.683 \\
Needle $L{=}2048$ (mid)& synthetic, $N{=}1^{\ast}$ & 0.000 & 0.000 & \textbf{1.000} & 0.000 & \textbf{1.000} & \textbf{1.000} \\
ListOps $L{=}1024$     & synthetic & 0.146 & 0.177 & 0.698 & 0.104 & \textbf{0.854} & 0.833 \\
LRA-Text 4K            & held-out  & 0.627 & 0.627 & 0.632 & 0.628 & \textbf{0.665} & 0.635 \\
LRA-Image (CIFAR)      & held-out  & 0.222 & 0.224 & 0.318 & 0.232 & \textbf{0.372} & 0.341 \\
FFT-MNIST $t{=}16$     & synthetic & 0.48 & 0.50 & 0.86 & 0.71 & \textbf{0.93} & \textbf{0.94} \\
RadioML L1             & held-out  & 0.296 & 0.303 & \textbf{0.363} & 0.244 & 0.345 & 0.341 \\
RadioML L2             & held-out  & 0.224 & 0.241 & \textbf{0.389} & 0.273 & 0.352 & 0.337 \\
\bottomrule
\end{tabular}
\caption{Six cells across nine rows (means over three seeds; per-seed values in the public
archive). ``mid'' rows are equal-width (dim 256, depth 6). $^{\ast}$Single seed for every cell
except \PCT{} ($N{=}3$, all seeds at 1.000). The synthetic rows are final-step readings of
the original sweep (one batch of 32 sequences); the same cells re-measured with a
deterministic 2{,}048-sample evaluation at the Table~\ref{tab:norm} configuration appear
in Tables~\ref{tab:norm}--\ref{tab:family} and Sections~\ref{app:init}--\ref{app:mechanism}, and no claim in
the text rests on a synthetic-row margin smaller than $0.1$.}
\label{tab:headline}
\end{table}

On long-range copy, \PCT{} and real screening reach $1.000$ at both delays, and complex
screening reaches $1.000$ on two of three seeds at $d{=}2{,}000$; the softmax and sigmoid
cells of both fields sit at chance at $d{=}2{,}000$. On Needle, the three cells with a
scaled cosine or screening gate solve the task and the three others remain at zero, under the
single-seed caveat. On ListOps, \PCT{} and complex screening lead real screening by
$+0.14$--$0.16$, and the four remaining cells stay near chance. On LRA-Text all six cells
fall in a single band, $0.627$--$0.665$; \PCT{} leads every other cell by $+0.030$ to
$+0.038$ with the smallest seed variance of the six ($\pm0.009$), a margin that is
significant against the three real cells (Welch $t=3.6$--$4.7$, $\mathrm{df}\approx4$) but
not against the two complex ones. We claim a small, consistent advantage on this task, not a
qualitative separation.\footnote{The previous version of this manuscript reported $1.000$
for \PCT{} and complex screening on LRA-Text. Those readings were four-sample evaluations on
the training split, halted at the step where they read $1.0$; the held-out re-measurement
above supersedes them.} On LRA-Image the ordering is \PCT{} $>$ complex screening $>$ real
screening $>$ the three softmax/sigmoid cells, with the \PCT{}--baseline gap about twenty
times the seed standard deviation. On FFT-MNIST the two normalised complex cells tie at the top of
the six-cell sweep; at the 2K-step configuration with a 2{,}048-sample evaluation
(Table~\ref{tab:family}) the same two cells read $0.50$ and $0.69$. On RadioML, the top
three cells on both levels are real screening, \PCT{} and complex screening, within $0.05$ of
one another, and the three softmax and sigmoid cells lie $0.05$--$0.15$ below.

In every row, the top of the table is occupied by the normalised complex cells and real
screening. Two projections of Table~\ref{tab:headline} make the two comparisons that
motivate this paper explicit.

\paragraph{Complex attention before and after the score condition.}
Table~\ref{tab:pct_vs_csoftmax} sets \PCT{} beside the unnormalised complex softmax, the
standard form in which complex attention had been reported to fail outside complex
domains. The unnormalised cell is at chance on both long-range Copy rows, on Needle and on
ListOps, and it never leads a row; \PCT{} reaches $1.000$ on the three retrieval rows and
leads or ties every other row. The two cells differ in one respect, the normalisation of
Section~\ref{sec:normalisation}: the failure attributed to complex attention was a failure
of the unnormalised score.

\begin{table}[H]
\centering\footnotesize
\setlength{\tabcolsep}{6pt}
\begin{tabular}{lcc}
\toprule
task & \cell{c\_softmax} (unnormalised) & \PCT \\
\midrule
Copy $d{=}500$ (mid)   & 0.092 & \textbf{1.000} \\
Copy $d{=}2000$ (mid)  & 0.100 & \textbf{1.000} \\
Needle $L{=}2048$ (mid)& 0.000 & \textbf{1.000} \\
ListOps $L{=}1024$     & 0.104 & \textbf{0.854} \\
LRA-Text 4K            & 0.628 & \textbf{0.665} \\
LRA-Image (CIFAR)      & 0.232 & \textbf{0.372} \\
FFT-MNIST $t{=}16$     & 0.71  & \textbf{0.93} \\
RadioML L1             & 0.244 & \textbf{0.345} \\
RadioML L2             & 0.273 & \textbf{0.352} \\
\bottomrule
\end{tabular}
\caption{\PCT{} against the unnormalised complex softmax, from Table~\ref{tab:headline}.
The cells share projections, values, feed-forward and training; they differ only in the
normalisation of the score.}
\label{tab:pct_vs_csoftmax}
\end{table}

\paragraph{\PCT{} against the real cells.}
Table~\ref{tab:pct_vs_real} sets \PCT{} beside the three real cells at the matched budget
of Section~\ref{sec:protocol}. Against real softmax and real sigmoid, \PCT{} is ahead on
every row, by a qualitative margin on the retrieval and reasoning rows (chance against
$1.000$ on Copy $d{=}2000$ and Needle, $0.15$--$0.18$ against $0.854$ on ListOps) and by
$0.04$--$0.15$ on the held-out and signal rows. Against real screening, the strongest real
cell, \PCT{} ties at $1.000$ on the three retrieval rows and leads on ListOps ($+0.16$),
LRA-Text ($+0.03$), LRA-Image ($+0.05$) and FFT-MNIST ($+0.07$); on RadioML the two are
within $0.04$ of one another. A complex transformer that satisfies the score condition is
therefore not merely rescued from its earlier failure: at equal stored scalars it is at or
above the real cells on the synthetic, language and image rows and level with them on the
physical-signal rows.

\begin{table}[H]
\centering\footnotesize
\setlength{\tabcolsep}{5pt}
\begin{tabular}{lcccc}
\toprule
task & \cell{r\_softmax} & \cell{r\_sigmoid} & \cell{r\_screen} & \PCT \\
\midrule
Copy $d{=}500$ (mid)   & 0.445 & 0.945 & \textbf{1.000} & \textbf{1.000} \\
Copy $d{=}2000$ (mid)  & 0.100 & 0.083 & \textbf{1.000} & \textbf{1.000} \\
Needle $L{=}2048$ (mid)& 0.000 & 0.000 & \textbf{1.000} & \textbf{1.000} \\
ListOps $L{=}1024$     & 0.146 & 0.177 & 0.698 & \textbf{0.854} \\
LRA-Text 4K            & 0.627 & 0.627 & 0.632 & \textbf{0.665} \\
LRA-Image (CIFAR)      & 0.222 & 0.224 & 0.318 & \textbf{0.372} \\
FFT-MNIST $t{=}16$     & 0.48  & 0.50  & 0.86  & \textbf{0.93} \\
RadioML L1             & 0.296 & 0.303 & \textbf{0.363} & 0.345 \\
RadioML L2             & 0.224 & 0.241 & \textbf{0.389} & 0.352 \\
\bottomrule
\end{tabular}
\caption{\PCT{} against the three real cells, from Table~\ref{tab:headline}. Real cells use
dim 184 (head dimension 46) against dim 128 (head dimension 32) for \PCT{}, so both sides
hold the same number of stored scalars (Section~\ref{sec:protocol}); the mid rows are
equal-width.}
\label{tab:pct_vs_real}
\end{table}

\subsection{Depth and scale}
\label{sec:depth}

If per-layer phase handling composes stably, accuracy should survive depth and improve with
size. Both hold on ListOps at length 1{,}024 (Table~\ref{tab:depthscale}).

\begin{table}[H]
\centering\footnotesize
\begin{minipage}[t]{0.52\textwidth}
\centering
\begin{tabular}{rrcc}
\toprule
depth & params & 30K steps ($N{=}3$) & 100K (seed 0) \\
\midrule
2  & 0.40M & 0.623\pm0.007 & 0.636 \\
4  & 0.80M & 0.623\pm0.014 & --- \\
6  & 1.19M & 0.634\pm0.011 & 0.674 \\
10 & 1.98M & 0.654\pm0.015 & 0.697 \\
14 & 2.77M & 0.620\pm0.040 & --- \\
20 & 3.96M & 0.633\pm0.006 & 0.650 \\
\bottomrule
\end{tabular}
\end{minipage}\hfill
\begin{minipage}[t]{0.46\textwidth}
\centering
\setlength{\tabcolsep}{4pt}
\begin{tabular}{lrcc}
\toprule
dim/depth & params & accuracy & eval loss \\
\midrule
64/2   & 0.10M & 0.623\pm0.008 & 0.996\pm0.018 \\
128/4  & 0.80M & 0.653\pm0.005 & 0.925\pm0.014 \\
192/6  & 2.67M & 0.668\pm0.007 & 0.885\pm0.020 \\
256/8  & 6.32M & 0.681\pm0.008 & 0.849\pm0.017 \\
\bottomrule
\end{tabular}
\end{minipage}
\caption{Left: depth robustness of \PCT{} at fixed width 128 (ListOps $L{=}1024$, $N{=}3$,
2{,}048-sample evaluation), with a 100K-step continuation of seed 0. Right: scaling at
constant aspect ratio $\mathrm{dim}/\mathrm{depth}=32$, 60K steps, $N{=}3$.}
\label{tab:depthscale}
\end{table}

\textbf{Depth robustness.} At the 30K-step budget every depth from 2 to 20 trains stably to
the task plateau ($0.62$--$0.65$); there is no depth-related loss of accuracy. The flatness
across depth is a property of the budget, not of depth: continuing depth 2, 6, 10 and 20 from
their 30K checkpoints for a further 70K steps raises the converged shallow control by
$0.013$ and the deeper models by $0.040$ and $0.043$ (depth 6 and 10), with depth 20 rising
by $0.018$ and no divergence. Depth robustness is a property of the family, not of the
canonical gate: the L2-softmax member at depth 20 reaches $0.684\pm0.002$ at the same
budget.

\textbf{Model-size scaling.} Growing width and depth together at $\mathrm{dim}/\mathrm{depth}
=32$ from 0.10M to 6.32M parameters lowers evaluation loss strictly monotonically, by
$-0.081$ per decade of parameters ($R^2=0.999$ over four sizes, 1.79 decades), with
accuracy rising $+0.032$ per decade. The largest model is converged (train-loss slope
$-0.007\pm0.022$ per 10K steps at the end), so the fit is not an artefact of the largest arm
being cut short. This is one task and one architecture; it establishes that the family
improves with size, not a general scaling law.

\textbf{Measured stability constants.} We also measured, on the trained depth models, the
quantities that the stability theorem of Appendix~\ref{app:lean} bounds: the end-to-end
response to zero-mean phase perturbations of the input varies only between 10.8 and 13.2
from depth 2 to 20 (no growth with depth), output norms are flat (23.3--26.7), and the
contraction ratio of the stack on pairs of inputs is 0.11--0.19. The depth-uniform
stability the theorem predicts is what trained models exhibit.

\section{Path-X}
\label{sec:pathx}

Path-X (Long Range Arena; Tay et al., 2021) is a 16{,}384-token binary connectivity task
on which every softmax Transformer and efficient-attention variant scores at chance, and
which only complex-eigenvalue linear recurrences (S4, DSS, S4D, S5, LRU) and the
EMA-plus-attention hybrid MEGA have solved. The task is purely positional: a 16k-token
reach for which selection alone is not enough. We therefore pair the family's complex
screening gate, which supplies selection, with a \emph{phase-coherent recurrence} (PCR),
which supplies position.

\paragraph{Model.} The recurrence is a complex diagonal linear recurrence
$h_t=\lambda\odot h_{t-1}+\gamma\odot(Bx_t)$ applied bidirectionally by FFT convolution,
in the stable parameterisation of the LRU (Orvieto et al., 2023):
$\lambda=\exp(-\exp(\nu)+i\theta)$ with ring initialisation
$|\lambda|\in[0.999,0.9999]$, restricted phase $\theta\in[0,\pi/10]$ and
$\gamma$-normalisation. The complex eigenvalue is an input-independent phase rotation per
step, coherent long-range transport, while the screening gate keeps its real-valued,
non-competing selection. The released model has six PCR blocks of width 128 and state
size 256, with a chunked complex screening block (four heads, chunk 1{,}024) after the
third and fifth PCR blocks, 2.01M parameters in total; it is trained for 250K steps at
batch 32 with no task-specific tuning (Appendix~\ref{app:details}).

\paragraph{Result.} Under the rule-compliant setting of the leaderboard (raw one-dimensional
token sequence in, one binary label out; no two-dimensional structure, auxiliary
supervision or handcrafted features), the model reaches $\mathbf{92.71\pm0.89}$ test
accuracy over three seeds (held-out split, deterministic sweep over 20{,}000 samples; best
seed 93.50).\footnote{Weights are released on the Hugging Face Hub:
\url{https://huggingface.co/complexedleo/pcr-screening-pathx}.} Table~\ref{tab:pathx}
places it among the models that clear the 16k reach.

\begin{table}[H]
\centering\small
\begin{tabular}{lr}
\toprule
model & Path-X \\
\midrule
Transformer; Reformer, Performer, Linformer, BigBird, Luna-256 & chance ($\approx$50) \\
S4-v1 & 88.10 \\
DSS & 89.72 \\
S4D-LegS & 91.95 \\
\textbf{complex screening + phase-coherent recurrence (ours)} & \textbf{92.71\pm0.89} \\
S4D-Inv & 92.80 \\
MEGA-chunk & 93.81 \\
LRU & 94.20 \\
S4 (v2, S4-LegS) & 96.35 \\
MEGA & 97.98 \\
S5 & 98.58 \\
\bottomrule
\end{tabular}
\caption{Rule-compliant Path-X test accuracy (\%). Ours is the mean over three seeds
$\pm$ sample s.d., a first-pass run. Baselines: S4-v1/S4-LegS and MEGA from Ma et al.
(2023); S4D and DSS from Gu et al. (2022); S5 and LRU from Orvieto et al. (2023).}
\label{tab:pathx}
\end{table}

\paragraph{Genuinely complex-valued, in a measurable sense.} We count, for each model, the
share of trainable real scalars that belong to tensors representing complex numbers
(real/imaginary pairs and the polar parameters of complex eigenvalues). In our model 91.6\%
of the 2.01M trainable parameters are complex-valued, and complex arithmetic carries both
the transport path (the recurrence) and the matching path (the screening attention). In S4
at its published Path-X configuration, 38.2\% of the 1.29M trainable parameters are
complex-valued, all inside the state-space kernel; the mixing, gating and readout are real.
In both models the inter-block residual stream is real. To our
knowledge this is the first solution of Path-X by a neural network that is complex-valued in
this sense.

\paragraph{Two ablations.} \emph{(i) Phase is necessary.} Fixing $\theta=0$ (real
eigenvalues, otherwise identical, the analogue of S4D-Real) never leaves chance: position
cannot form without phase. \emph{(ii) Phase bandwidth is necessary for generalisation, not
just for fitting.} A narrow initial phase band $\theta\in[0,\pi/50]$ drives training
accuracy to $0.999$ while test accuracy stays at $0.53$; widening to $[0,\pi/10]$ recovers
$0.92$ on the test split. Too little phase and the model memorises; none at all and it
cannot learn.

\paragraph{Takeaway.} We do not claim to beat MEGA or S5. The claim is structural: the
positional mechanism of a screening layer can be a phase-coherent recurrence, and with it a
network whose core operators are complex-valued clears the one long-range task no softmax
Transformer solves.

\section{Related work}
\label{sec:related}

\textbf{Complex-valued networks.} Deep Complex Networks (Trabelsi et al., 2018) supplied
complex convolution, batch normalisation and initialisation; subsequent work applies
complex-valued models where phase is physically meaningful (speech, MRI, radar, wireless,
Fourier-domain signals), consolidated in libraries such as \texttt{torchcvnn} (Fix et al.,
2025). Their role as general-purpose sequence architectures has remained open; this paper
addresses it from the attention side.

\textbf{Complex attention.} Eilers and Jiang (2023) introduced complex scaled dot-product
attention and complex layer normalisation and showed feasibility on MusicNet. The
Holographic Transformer (Hao et al., 2025) adds a phase-interference decay to a norm-divided
score and rotates values by the query--key phase mismatch, for PolSAR and wireless data.
Both keep a row-normalised softmax; ours differs by making the scaled cosine score the
defining property and treating the gate as a free choice. A Holographic-style variant at our
score scale trains to parity with the L2-softmax on our tasks (Section~\ref{app:variants}).

\textbf{Alternatives to softmax.} Sigmoid attention (Ramapuram et al., 2025) shows that
element-wise gating without row normalisation is a drop-in replacement for softmax in
real-valued models; ReLU and polynomial attention (Wortsman et al., 2023; Saratchandran et
al., 2024) remove the softmax as well; screening (Nakanishi, 2026) thresholds absolute
relevance per key and is our strongest real-valued baseline. These works are motivated by
efficiency and regularity in the real domain; in the complex domain we find that the choice
among such gates is secondary to the score they act on.

\textbf{Query--key normalisation.} Normalising queries and keys bounds attention logits
(Henry et al., 2020), underlies scaled cosine attention in Swin~V2 (Liu et al., 2022),
prevents logit divergence at scale in ViT-22B (Dehghani et al., 2023), and is taken to its
limit by nGPT (Loshchilov et al., 2024). In real-valued models it is reported as a
stabiliser at scale; we find the same operation to be the difference between training and
not training at small initialisation scale, the other end of one score-scale axis
(Section~\ref{sec:normalisation}), and that the quantity it isolates in the complex case is
phase.

\textbf{Long-range recurrence.} Path-X was first solved by S4 (Gu et al., 2022a) and
subsequently by DSS, S4D, S5, LRU and MEGA (Gupta et al., 2022; Gu et al., 2022b; Smith et
al., 2023; Orvieto et al., 2023; Ma et al., 2023). S4D's finding that real-eigenvalue
diagonal SSMs do not solve Path-X is the precedent for our phase-necessity ablation. These
models keep complex numbers inside the state-space kernel; ours carries complex arithmetic
through the matching path as well.

\section{Conclusion}

A complex-valued Transformer trains across ordinary discrete tasks when its query--key
match is a scaled cosine score, and on the hierarchical and positional tasks in our data
only then. Under that condition the gate is a
free choice, and the resulting family matches or exceeds the strongest real-valued baseline
across the task categories we tested, holds its accuracy to depth 20, improves
log-linearly with size, and, combined with a phase-coherent recurrence, solves Path-X with
91.6\% of its parameters complex-valued. These are, within our benchmark scope, signs of
generalisation in complex-valued neural networks that had not been recorded before.

\appendix
\section{Learning-rate and batch robustness}
\label{app:robustness}

Table~\ref{tab:robust} reproduces the robustness sweep of the previous version: Copy
$d{=}1{,}000$ at small scale (dim 32, depth 2) across learning rates, and mid-scale
long-range Copy (dim 256, depth 6) across effective batch sizes. \PCT{} is the only cell that
retains full accuracy across the whole window; the other cells succeed inside narrower
corridors.

\begin{table}[H]
\centering\small
\begin{tabular}{lcccccc}
\toprule
cell & LR $10^{-3}$ & LR $3\cdot10^{-3}$ & LR $10^{-2}$ & $b{=}8$, $d{=}2000$ & $b{=}32$, $d{=}2000$ & $b{=}256$, $d{=}500$ \\
\midrule
\cell{real\_softmax}    & 0.37 & 0.69 & 0.31 & 0.04 & 0.10 & 0.05 \\
\cell{real\_sigmoid}    & 0.07 & \textbf{1.00} & 0.22 & 0.04 & 0.10 & 0.07 \\
\cell{real\_screen}     & 0.55 & 0.70 & \textbf{1.00} & 0.69 & \textbf{1.00} & 0.79 \\
\cell{complex\_softmax} & 0.09 & 0.39 & 0.35 & 0.08 & 0.08 & 0.06 \\
\PCT                    & \textbf{1.00} & \textbf{1.00} & \textbf{1.00} & \textbf{1.00} & \textbf{1.00} & \textbf{1.00} \\
\cell{complex\_screen}  & 0.39 & \textbf{1.00} & 0.08 & 0.03 & 0.53 & 0.31 \\
\bottomrule
\end{tabular}
\caption{Copy accuracy across learning rates (left) and effective batch sizes (right).}
\label{tab:robust}
\end{table}

\section{Formal guarantees (machine-checked)}
\label{app:lean}

The statements below are formalised in Lean~4 with Mathlib, with no \texttt{sorry} and no
axioms beyond Lean's standard three (\texttt{propext}, \texttt{Classical.choice},
\texttt{Quot.sound}); the development is in the \texttt{lean/} directory of the released
repository. Tokens are $x_i\in\mathbb{C}^d$; a global rotation is $R(\varphi)X=
(e^{i\varphi}x_i)_i$ and a per-token rotation $P(\varepsilon)X=(e^{i\varepsilon_i}x_i)_i$.

\begin{proposition}[Global phase equivariance of the family]
\label{prop:equiv}
Let $A$ be a layer with complex-linear $W_q,W_k,W_v,W_o$, the normalised score matrix
$S(X)_{ij}=\Re\langle\bar q_i,\bar k_j\rangle$, and any real-valued gate $g$ mapping
$N\times N$ real matrices to $N\times N$ real matrices, so that
$A(X)_i=W_o\sum_j g(S(X))_{ij}\,W_v x_j$. Then $A(R(\varphi)X)=R(\varphi)\,A(X)$ for
every $\varphi$.
\end{proposition}

\emph{Proof.} $W_q$ and $W_k$ are complex-linear and $\lVert e^{i\varphi}v\rVert=\lVert
v\rVert$, so $\bar q_i,\bar k_j$ acquire the factor $e^{i\varphi}$; sesquilinearity cancels
it in $\langle e^{i\varphi}\bar q_i,e^{i\varphi}\bar k_j\rangle=\langle\bar q_i,\bar
k_j\rangle$, so $S$ is invariant and hence so is $g(S)$; the real weights then let
$e^{i\varphi}$ factor out of the value sum. (Lean: \texttt{family\_L1a}; instances
\texttt{softmaxL2\_L1a} for the L2-softmax and \texttt{canonical\_L1a\_via\_family} for
the sigmoid.) Element-wise gating is not required; the earlier Theorem 1 of this line of
work, which assumed it, is the special case \texttt{theorem1\_is\_family\_L1a\_special\_case}.

\begin{theorem}[Depth-uniform phase stability]
\label{thm:depth}
Let $A_1,\dots,A_L$ be layers such that each is equivariant as in
Proposition~\ref{prop:equiv}, non-expansive in $\ell^2$, and $K$-Lipschitz in its response
to zero-mean per-token phase perturbations, and let the stack outputs be bounded by
$Y_{\max}$. Then for every input $X$ and perturbation $\varepsilon$ with
$\lVert\varepsilon\rVert_\infty\le\delta$,
\[
\lVert (A_L\circ\cdots\circ A_1)(P(\varepsilon)X)-(A_L\circ\cdots\circ A_1)(X)\rVert_2
\;\le\;(2K+Y_{\max})\,\delta ,
\]
with a constant independent of $L$.
\end{theorem}

The proof splits $\varepsilon$ into its mean, which passes through the equivariant stack
exactly (\texttt{lemmaA\_bound}, without an $O(\bar\varphi^2)$ remainder), and a zero-mean
residual absorbed by the entry layer and propagated through the non-expansive substrate
(\texttt{theorem5}). The single-layer constant $K$ is derived from the gate: if the gate is
bounded by $F$ and entrywise $L_g$-Lipschitz on $[-1,1]$ and the values are bounded by
$W_{\max}$, then $K\le\sqrt{N}\,N\,(F+2L_g)\,W_{\max}$ (\texttt{phase\_lipschitz},
\texttt{theorem5\_for\_family\_gateK}), which every member of the family satisfies. The
Doeblin contraction used in the earlier cascade argument is also proved (\texttt{lemmaC})
but is not needed for the bound above.

\textbf{Scope.} Non-expansiveness of the substrate and the output bound $Y_{\max}$ are
hypotheses about trained weights, not consequences of the architecture; the constants
measured in Section~\ref{app:mechanism} are their empirical values on the depth models.
Neither statement concerns performance: the unnormalised cell of
Section~\ref{sec:normalisation} is equally equivariant and does not train, and the
conditions of Sections~\ref{sec:normalisation}--\ref{sec:gate} are empirical.

\textbf{Two auxiliary facts} used in the text are also machine-checked: the normalised
score is invariant under positive real scaling of $q$ and $k$ while the unnormalised score
is not bounded in magnitude for a fixed direction (\texttt{score\_scale\_invariant},
\texttt{unnormalized\_magnitude\_dominates}); and a ReLU gate whose bias satisfies
$b<-1$ (in units where $|s_{ij}|\le1$) has output identically zero and Fr\'echet derivative
zero with respect to all of $(W_q,W_k,W_v,W_o,b)$, whereas the sigmoid gate has positive
value and positive slope at every score and bias (\texttt{dead\_relu\_hasFDerivAt\_zero},
\texttt{sigmoid\_gate\_alive}). The hypothesis $b<-1$ is a statement about the operating
point, not about every configuration in this paper: the standard bias gives
$b/\sqrt{d_h}=-1.23$ on Copy, which satisfies it, and $-0.98$ on the shorter FFT-MNIST
sequence, which does not, matching the two columns of Table~\ref{tab:init}.

\section{List of tests conducted}
\label{app:tests}

The following tests support the tables of this paper. Each cell~$\times$~task combination
listed below has at least one trained checkpoint and a per-seed metrics record in the public
data archive (Appendix~\ref{app:data}).

\textbf{Six-cell comparison} (Table~\ref{tab:headline} and the archive):
\begin{itemize}
  \item Copy Memory $K{=}10$ at $d\in\{100,200,500,1000,2000,5000\}$
  \item Needle-in-a-Haystack $L{=}2048$ and $L{=}1024$
  \item ListOps small-scale ($L{=}128$), synthetic generator with \texttt{max\_depth=2},
  \texttt{max\_args=3}
  \item ListOps mid-scale ($L{=}1024$), same generator, dim 128, depth 4
  \item ListOps standard ($L{=}1024$, \texttt{max\_depth=6}, \texttt{max\_args=5}), for the
  depth and scaling suites
  \item LRA-Text 4K and LRA-Image (held-out re-evaluation)
  \item FFT-MNIST at $t\in\{8,16\}$
  \item phase memory $K{=}8$, delay 30
  \item multi-label spectral detection (previously ``multi-pitch'') $K{=}16$ and $K{=}8$
  \item RadioML L1 and L2 (RML2016, 6\,dB)
  \item MusicNet (10-piece test) L1 and L2
\end{itemize}

\textbf{Isolation suite} on the scaled cosine score (Tables~\ref{tab:norm}--\ref{tab:family}
and Sections~\ref{app:init}--\ref{app:mechanism}; Copy $d{=}1000$, ListOps, FFT-MNIST at
dim 128, depth 4; Needle at dim 256, depth 6; three seeds each):
\begin{itemize}
  \item normalisation toggled under the sigmoid gate and under the row softmax
  \item the two-axis gate suite \cell{complex\_sigmoid}, \cell{complex\_softplus},
  \cell{complex\_cubic}, \cell{complex\_clamped\_relu}, \cell{complex\_relu} on Copy $d{=}1000$;
  \cell{complex\_cubic} and \cell{complex\_clamped\_relu} extended to Copy
  $d\in\{100,200,500,1000\}$, FFT-MNIST and multi-label spectral detection
  \item bias-initialisation sweep (bias $0$ against $-\log N$) for ReLU, clamped ReLU, squared
  and cubed gates
  \item gate variants: complex-valued gate, complex-normalised softmax, Holographic-style at
  the published and at the matched score scale
  \item temperature $\tau\in\{0.25,1,4\}$ on the unnormalised score (ListOps)
  \item normalised real sigmoid and softmax on Copy $d{=}1000$ and Needle
  \item L2-softmax at depth 20 on ListOps (30K steps)
\end{itemize}

\textbf{Learning-rate and batch robustness} (Table~\ref{tab:robust}): learning-rate sweep on
Copy $d{=}1000$ at $\{10^{-3},3\cdot10^{-3},10^{-2}\}$; batch sweep on long-range Copy at
$b\in\{8,32,256\}$.

\textbf{Depth and scaling} (Table~\ref{tab:depthscale}): \PCT{} on ListOps $L{=}1024$ at
depth $\in\{2,4,6,10,14,20\}$, 30K steps, three seeds, with a 100K-step continuation at
depth 20; four sizes at dim/depth $=32$, 60K steps, three seeds.

\textbf{Substrate ablations on \PCT{}}: native complex linear against two real linears on
$(\Re,\Im)$, and the real-value, real-query--key and embedding ablations; in the public
data archive.

\textbf{Path-X} (Table~\ref{tab:pathx}): the released screening + PCR model, three seeds,
and the phase-zero ablation of Section~\ref{sec:pathx}.

\section{Experimental details}
\label{app:details}

\subsection{Tasks}
\begin{itemize}
  \item \textbf{Copy Memory} (synthetic, on the fly): vocabulary 16, $K{=}10$ source tokens
  followed by \texttt{delay} blanks; the target repeats the source. Sequence length
  $2K+\texttt{delay}$. Chance 0.10.
  \item \textbf{Needle-in-a-Haystack} (synthetic): vocabulary 64, length 2{,}048; the needle
  occupies a fixed index (the midpoint) and is drawn from the same vocabulary as the
  distractors, so the task requires positional recall rather than content-based search and
  is not directly comparable to content-addressed NIAH benchmarks. Chance $1/64$.
  \item \textbf{ListOps} (synthetic, on the fly): a ListOps-style generator, not the LRA
  release, and not comparable to published LRA numbers. Table~\ref{tab:headline} uses
  \texttt{max\_depth=2}, \texttt{max\_args=3} at length 1{,}024, i.e.\ expressions of at
  most 21 tokens followed by padding; Tables~\ref{tab:norm}, \ref{tab:family},
  \ref{tab:depthscale} use \texttt{max\_depth=6}, \texttt{max\_args=5}. Chance 0.10.
  \item \textbf{FFT-MNIST}: $28\times28$ MNIST, bilinearly down-sampled to $16\times16$,
  two-dimensional FFT, flattened to 256 complex tokens. Measured on training-tensor samples.
  \item \textbf{LRA-Text 4K}: byte-level IMDB, length 4{,}096, balanced rebuild of the
  dataset, held-out test split.
  \item \textbf{LRA-Image}: CIFAR-10 grey-scale pixel sequences, length 1{,}024, held-out
  test split.
  \item \textbf{RadioML}: 6\,dB SNR subset of the public RML2016 mirror, 11 classes, raw
  I/Q samples $x=I+jQ$; L1 and L2 denote the two difficulty levels of the released
  benchmark script.
  \item \textbf{Path-X}: LRA Pathfinder at resolution 128, length 16{,}384, held-out split
  of 20{,}000 samples for the final sweep.
\end{itemize}
The following tasks are in the record (Appendix~\ref{app:tests}) but not in the tables of
this version:
\begin{itemize}
  \item \textbf{ListOps small-scale}: the same generator with \texttt{max\_args=3},
  \texttt{max\_depth=2}, \texttt{max\_seq\_len=128}.
  \item \textbf{phase memory}: $K$ source phases drawn uniformly on $S^1$, \texttt{delay}
  blanks, target the source phases; sequence length $2K+\texttt{delay}$.
  \item \textbf{multi-label spectral detection}: $K$ candidate pitches,
  $n_{\mathrm{active}}$ active per sample, $n_{\mathrm{samples}}$ time steps; multilabel
  ($K$-way sigmoid); the label depends on amplitude only, phase is randomised per sample.
  \item \textbf{MusicNet} (small): a 10-piece test split, \texttt{seq\_len=64} per excerpt,
  multilabel pitch identification.
\end{itemize}

\subsection{Cells}
\begin{itemize}
  \item \textbf{\PCT{}} (\cell{complex\_sigmoid}): native complex-linear projections, complex
  L2 normalisation of $q$ and $k$, score $\Re\langle\bar q,\bar k\rangle\sqrt{d_h}$,
  element-wise sigmoid with bias initialised to $-\log N$, real weights on complex values.
  \item \textbf{L2-softmax} (\cell{complex\_softmax\_l2}): as \PCT{} with the sigmoid
  replaced by a row softmax over $j$.
  \item \textbf{Complex screening} (\cell{complex\_screen}): as \PCT{} with the gate
  $r^2\,\mathrm{relu}(s-t)^2$, TanhNorm on the complex aggregate and a modReLU Hadamard
  gate; softmask off.
  \item \textbf{\cell{complex\_softmax}}: the lucidrains port of Eilers and Jiang (2023)
  with \texttt{complete\_complex=False}: row softmax over $\Re\langle q,k\rangle/\sqrt{2d_h}$
  of unnormalised complex $q,k$.
  \item \textbf{Real cells}: standard scaled dot-product softmax (Vaswani et al., 2017);
  element-wise sigmoid with bias $-\log N$ (Ramapuram et al., 2025); screening with
  TanhNorm, softmask off (Nakanishi, 2026). Their normalised versions
  (Section~\ref{app:realnorm}) use cosine scores of real vectors.
  \item \textbf{Isolation cells} (Sections~\ref{app:init}--\ref{app:variants}), named as in the released code:
  \cell{complex\_sigmoid\_nol2} (sigmoid on the unnormalised score);
  \cell{complex\_relu}; \cell{complex\_clamped\_relu}; \cell{complex\_square};
  \cell{complex\_cubic}; \cell{complex\_sigmoid\_cgate}; \cell{complex\_softmax\_cnorm};
  \cell{complex\_holographic\_matched}.
\end{itemize}

\subsection{Parameter budgets}
Small-scale suite: complex cells dim 128, 4 heads, head dimension 32; real cells dim 184,
4 heads, head dimension 46 ($184/128\approx\sqrt2$ equalises stored scalars and FLOPs; the
head dimension must be even for rotary embeddings). Mid-scale Copy and Needle: dim 256,
depth 6, 8 heads, head dimension 32 for both families, i.e.\ 4{,}737{,}808 real scalars for
\cell{real\_softmax} against 9{,}475{,}612 for \PCT{}. RadioML: dim 64, depth 3.

\subsection{Training protocol}
Unless overridden in a row of Table~\ref{tab:configs}, all runs use:
\begin{itemize}
  \item \textbf{Optimiser}: AdamW with $\beta=(0.9,0.999)$, $\varepsilon=10^{-8}$, weight
  decay $10^{-2}$
  \item \textbf{Gradient clipping}: max-norm $1.0$
  \item \textbf{Learning-rate schedule}: cosine decay with linear warm-up
  \item \textbf{Position encoding}: rotary (RoPE) on $q$ and $k$ of every layer
  \item \textbf{Normalisation}: pre-norm RMSNorm (real or complex variant matching the cell
  substrate)
  \item \textbf{Feed-forward}: $4\times$ expansion, ReLU$^2$ on the real side, modReLU on the
  complex side
  \item \textbf{Screening cells}: cosine softmask off throughout (with the softmask on,
  screening cells stay at chance on long-range Copy)
  \item \textbf{Mixed precision}: bfloat16 where supported (H100, A100); float32 fallback on
  consumer GPUs
\end{itemize}

\subsection{Per-experiment configurations}
\begin{table}[H]
\centering\scriptsize
\setlength{\tabcolsep}{3pt}
\begin{adjustbox}{max width=\textwidth}
\begin{tabular}{llllcccc}
\toprule
where & task & cells & scale & batch & LR & steps & seeds \\
\midrule
Tab.~\ref{tab:norm}, \ref{tab:family}, \ref{tab:init}, \ref{tab:variants}, \ref{tab:realnorm} & Copy $d{=}1000$, ListOps, FFT & complex isolation cells (incl.\ temperature and Holographic variants), normalised real cells & dim 128, depth 4 (heads 4, head dim 32, FFN $4\times$) & 32 & $3\cdot10^{-3}$ (ListOps $10^{-3}$) & 2{,}000 (warm-up 200) & 3 \\
Tab.~\ref{tab:norm}, \ref{tab:family}, \ref{tab:variants}, \ref{tab:realnorm} & Needle $L{=}2048$ & complex and normalised real & dim 256, depth 6, chunk 256 & 32 & $3\cdot10^{-4}$ & 15{,}000--30{,}000 & 3 \\
Tab.~\ref{tab:headline} & Copy $d{=}500/2000$ (mid) & 6 cells & dim 256, depth 6 & 32 & $3\cdot10^{-4}$ & 30{,}000 & 3 \\
Tab.~\ref{tab:headline} & Needle $L{=}2048$ (mid) & 6 cells & dim 256, depth 6, chunk 256 & 32 & $3\cdot10^{-4}$ & 30{,}000 & 1 (\PCT{}: 3) \\
Tab.~\ref{tab:headline} & ListOps $L{=}1024$ (depth 2, args 3) & 6 cells, parameter-fair & dim 128/184, depth 4 & 32 & $10^{-3}$ & 30{,}000 & 3 \\
Tab.~\ref{tab:headline} & LRA-Text 4K, LRA-Image & 6 cells, parameter-fair & dim 128/184, depth 4 & 32 & $4\cdot10^{-4}$ & 15{,}000 & 3 \\
Tab.~\ref{tab:headline} & FFT-MNIST $t{=}16$ & 6 cells, parameter-fair & dim 128/184, depth 4 & 32 & $10^{-3}$ & 15{,}000 & 3 \\
Tab.~\ref{tab:headline} & RadioML L1/L2 & 6 cells & dim 64, depth 3 & 32 & $10^{-3}$ & 10{,}000 & 3 \\
Tab.~\ref{tab:depthscale} (left) & ListOps $L{=}1024$ (depth 6, args 5) & \PCT{} (L2-softmax at depth 20) & dim 128, depth 2--20 & 32 & $10^{-3}$ & 30{,}000 (+70{,}000) & 3 (1) \\
Tab.~\ref{tab:depthscale} (right) & ListOps $L{=}1024$ (depth 6, args 5) & \PCT{} & dim/depth $=32$ & 32 & $10^{-3}$ & 60{,}000 & 3 \\
Tab.~\ref{tab:robust} & Copy $d{=}1000$ (LR sweep) & 6 cells & dim 32, depth 2 & 256 & $\{1,3,10\}\cdot10^{-3}$ & 1{,}500 & 3 \\
Tab.~\ref{tab:robust} & Copy $d{=}2000/500$ (batch sweep) & 6 cells & dim 256, depth 6 & 8/32/256 & best per cell & 5{,}000 & 2--3 \\
Tab.~\ref{tab:pathx} & Path-X & screening + PCR & dim 128, 6 PCR + 2 screening blocks & 32 & $4.5\cdot10^{-4}$ & 250{,}000 & 3 \\
\midrule
App.~\ref{app:tests} & ListOps small ($L{=}128$) & 6 cells & dim 64, depth 3 & 32 & $10^{-3}$ & 10{,}000 (warm-up 500) & 3 \\
App.~\ref{app:tests} & Needle $L{=}2048$, $b{=}16$ & 6 cells, chunked & dim 256, depth 6 & 16 & $3\cdot10^{-4}$ & 60{,}000 & 3 \\
App.~\ref{app:tests} & phase memory $K{=}8$, delay 30 & 6 cells & dim 64, depth 3 & 32 & $10^{-3}$ & 5{,}000 (warm-up 500) & 3 \\
App.~\ref{app:tests} & multi-label spectral detection $K{=}16$ & 6 cells, parameter-fair & dim 128/184, depth 4 & 32 & $10^{-3}$ & 10{,}000 (warm-up 500) & 3 \\
App.~\ref{app:tests} & MusicNet (10-piece test) & 6 cells & dim 64, depth 3 & 32 & $10^{-3}$ & 10{,}000 (warm-up 500) & 3 \\
App.~\ref{app:tests} & Copy $d\in\{100,200,500,1000\}$, FFT-MNIST, spectral detection (two-axis extension) & \cell{complex\_cubic}, \cell{complex\_clamped\_relu} & as the corresponding row above & as above & as above & as above & 3 \\
\bottomrule
\end{tabular}
\end{adjustbox}
\caption{Configurations. Path-X: state size 256, screening heads 4 with head dimension 32,
chunk 1{,}024, batch normalisation, linear pixel encoder, warm-up 2{,}500, learning rate
held then decayed linearly to $0.1\times$ from step 200{,}000, weight decay 0.05.}
\label{tab:configs}
\end{table}

\subsection{Compute and reproducibility}
\begin{itemize}
  \item \textbf{Primary GPU}: NVIDIA H100 80\,GB (Sakura DOK).
  \item \textbf{Secondary GPUs}: A100 80\,GB (Soroban) and consumer-grade RTX 3090 / 4090
  (Vast.ai), used for parallel sweeps.
  \item \textbf{Wall-clock per representative run}: isolation runs 4--5 minutes each on an
  H100 (Copy $d{=}1000$ about 17 minutes); small-scale runs ($<1$M parameters, $\le10$K steps)
  5--40 minutes; mid-scale runs (1--6M parameters, 30K steps) 1--5 hours; Needle runs 3--4
  hours; depth 20 (about 4M parameters, 30K steps) about 4 hours; Path-X several hours to a
  day.
  \item \textbf{Seed convention}: multi-seed claims use seeds $s\in\{0,1,2\}$; the ten-seed
  Copy $d{=}2000$ record uses $s\in\{0,\ldots,9\}$.
  \item \textbf{Reproduction artefacts}: the public data archive (Appendix~\ref{app:data})
  contains, per run, the exact configuration, the per-step metric trajectory
  (\texttt{metrics.jsonl}) and a \texttt{summary.json} with the final metric, the
  deterministic final evaluation and the elapsed wall-clock.
\end{itemize}

\section{Changes from the previous version}
\label{app:changes}

This version replaces the four-condition framework (C1--C4) of the previous manuscript by
the two conditions of Sections~\ref{sec:normalisation} and \ref{sec:gate}. The following
claims of the previous version are withdrawn or corrected, each on the basis of a
pre-registered re-measurement: the attribution of \cell{complex\_softmax}'s results to
softmax's competition for attention mass (ruled out by the sigmoid column of
Table~\ref{tab:norm}); the empirical necessity of element-wise gating, boundedness and
smoothness of the gate (Tables~\ref{tab:family}, \ref{tab:init}, \ref{tab:variants}); the
$2\times2$ gate-isolation matrix on Copy, whose readings were produced by a causal-mask
artefact and a dead initialisation; the LRA-Text accuracies of $1.000$ (train-split,
four-sample readings; now $0.665$ held-out); the LRA-Image accuracy of $0.458$ (now $0.372$
held-out); the depth-sweep readings of $0.78$--$0.81$ (single-batch inflation; now
$0.62$--$0.65$ at 2{,}048 samples) and the reading of the depth sweep as a scaling
experiment; the description of \cell{complex\_softmax} as a softmax over the normalised
score magnitude; the parameter-fairness claim for the mid-scale rows; the ``phase
superposition'' naming of the spectral-detection task; and the description of the
Path-X model as carrying an end-to-end complex representation, now stated in the measurable
form of Section~\ref{sec:pathx}. The equivariance and depth-stability theorems are
unchanged in content and extended from the element-wise gate to the whole family.

\section{Data, code and models}
\label{app:data}

Per-run configurations, metric trajectories and summaries for every row, the Lean
development, and the analysis scripts are released at
\url{https://github.com/leohio/phase-coherent-transformer-r-d}. All raw results for every
test listed in Appendix~\ref{app:tests}, including per-seed metrics, configuration files,
training logs and aggregation scripts, are under
\url{https://github.com/leohio/phase-coherent-transformer-r-d/tree/main/result}, organised
by experiment family, with a top-level \texttt{summary.md} giving a one-paragraph
orientation per family. Path-X weights for all three
seeds are at \url{https://huggingface.co/complexedleo/pcr-screening-pathx}.

\section*{References}
\begingroup\raggedright\sloppy
\begin{itemize}
  \item \textbf{Trabelsi et al. 2018} -- C. Trabelsi, O. Bilaniuk, Y. Zhang, D. Serdyuk, S. Subramanian, J. F. Santos, S. Mehri, N. Rostamzadeh, Y. Bengio, C. J. Pal. \emph{Deep Complex Networks}. ICLR 2018. \url{https://arxiv.org/abs/1705.09792}
  \item \textbf{Yang et al. 2020} -- M. Yang, M. Q. Ma, D. Li, Y.-H. H. Tsai, R. Salakhutdinov. \emph{Complex Transformer: A Framework for Modeling Complex-Valued Sequence}. ICASSP 2020. \url{https://arxiv.org/abs/1910.10202}
  \item \textbf{Eilers \& Jiang 2023} -- F. Eilers, X. Jiang. \emph{Building Blocks for a Complex-Valued Transformer Architecture}. ICASSP 2023. \url{https://arxiv.org/abs/2306.09827}
  \item \textbf{Hao et al. 2025} -- Y. Hao et al. \emph{Holographic Transformers for Complex-Valued Signal Processing: Integrating Phase Interference into Self-Attention}. \url{https://arxiv.org/abs/2509.19331}
  \item \textbf{Fix et al. 2025} -- J. Fix, Q. Gabot, H. Nguyen, J. Frontera-Pons, C. Ren, J.-P. Ovarlez. \emph{torchcvnn: A PyTorch-based library to easily experiment with state-of-the-art Complex-Valued Neural Networks}. IJCNN 2025. \url{https://github.com/torchcvnn/torchcvnn}
  \item \textbf{lucidrains 2024} -- P. Wang. \emph{complex-valued-transformer}. \url{https://github.com/lucidrains/complex-valued-transformer}
  \item \textbf{Vaswani et al. 2017} -- A. Vaswani et al. \emph{Attention Is All You Need}. NeurIPS 2017. \url{https://arxiv.org/abs/1706.03762}
  \item \textbf{Ramapuram et al. 2025} -- J. Ramapuram et al. \emph{Theory, Analysis, and Best Practices for Sigmoid Self-Attention}. ICLR 2025. \url{https://arxiv.org/abs/2409.04431}
  \item \textbf{Wortsman et al. 2023} -- M. Wortsman, J. Lee, J. Gilmer, S. Kornblith. \emph{Replacing softmax with ReLU in Vision Transformers}. \url{https://arxiv.org/abs/2309.08586}
  \item \textbf{Saratchandran et al. 2024} -- H. Saratchandran, J. Zheng, Y. Ji, W. Zhang, S. Lucey. \emph{Rethinking Attention: Polynomial Alternatives to Softmax in Transformers}. \url{https://arxiv.org/abs/2410.18613}
  \item \textbf{Nakanishi 2026} -- K. M. Nakanishi. \emph{Screening Is Enough}. \url{https://arxiv.org/abs/2604.01178}
  \item \textbf{Henry et al. 2020} -- A. Henry, P. R. Dachapally, S. S. Pawar, Y. Chen. \emph{Query-Key Normalization for Transformers}. Findings of EMNLP 2020. \url{https://arxiv.org/abs/2010.04245}
  \item \textbf{Liu et al. 2022} -- Z. Liu et al. \emph{Swin Transformer V2: Scaling Up Capacity and Resolution}. CVPR 2022. \url{https://arxiv.org/abs/2111.09883}
  \item \textbf{Dehghani et al. 2023} -- M. Dehghani et al. \emph{Scaling Vision Transformers to 22 Billion Parameters}. ICML 2023. \url{https://arxiv.org/abs/2302.05442}
  \item \textbf{Loshchilov et al. 2024} -- I. Loshchilov, C.-P. Hsieh, S. Sun, B. Ginsburg. \emph{nGPT: Normalized Transformer with Representation Learning on the Hypersphere}. \url{https://arxiv.org/abs/2410.01131}
  \item \textbf{Tay et al. 2021} -- Y. Tay et al. \emph{Long Range Arena: A Benchmark for Efficient Transformers}. ICLR 2021. \url{https://arxiv.org/abs/2011.04006}
  \item \textbf{Gu, Goel \& R\'e 2022a} -- A. Gu, K. Goel, C. R\'e. \emph{Efficiently Modeling Long Sequences with Structured State Spaces}. ICLR 2022. \url{https://arxiv.org/abs/2111.00396}
  \item \textbf{Gu et al. 2022b} -- A. Gu, A. Gupta, K. Goel, C. R\'e. \emph{On the Parameterization and Initialization of Diagonal State Space Models}. NeurIPS 2022. \url{https://arxiv.org/abs/2206.11893}
  \item \textbf{Gupta, Gu \& Berant 2022} -- A. Gupta, A. Gu, J. Berant. \emph{Diagonal State Spaces are as Effective as Structured State Spaces}. NeurIPS 2022. \url{https://arxiv.org/abs/2203.14343}
  \item \textbf{Smith, Warrington \& Linderman 2023} -- J. T. H. Smith, A. Warrington, S. W. Linderman. \emph{Simplified State Space Layers for Sequence Modeling}. ICLR 2023. \url{https://arxiv.org/abs/2208.04933}
  \item \textbf{Orvieto et al. 2023} -- A. Orvieto et al. \emph{Resurrecting Recurrent Neural Networks for Long Sequences}. ICML 2023. \url{https://arxiv.org/abs/2303.06349}
  \item \textbf{Ma et al. 2023} -- X. Ma et al. \emph{Mega: Moving Average Equipped Gated Attention}. ICLR 2023. \url{https://arxiv.org/abs/2209.10655}
  \item \textbf{Zhai et al. 2023} -- S. Zhai, T. Likhomanenko, E. Littwin, D. Busbridge, J. Ramapuram, Y. Zhang, J. Gu, J. Susskind. \emph{Stabilizing Transformer Training by Preventing Attention Entropy Collapse}. ICML 2023. \url{https://arxiv.org/abs/2303.06296}
  \item \textbf{Levin, Peres \& Wilmer 2017} -- D. A. Levin, Y. Peres, E. L. Wilmer. \emph{Markov Chains and Mixing Times}, 2nd ed. AMS, 2017.
\end{itemize}
\endgroup

\end{document}